\documentclass{article}
\usepackage{iclr2026_conference, times}

\usepackage{amsmath, amsfonts, bm}

\def\eqref#1{equation~\ref{#1}}

\def\1{\bm{1}}

\DeclareMathAlphabet{\mathsfit}{\encodingdefault}{\sfdefault}{m}{sl}
\SetMathAlphabet{\mathsfit}{bold}{\encodingdefault}{\sfdefault}{bx}{n}

\usepackage{graphicx} 
\usepackage{hyperref, url}
\usepackage{amssymb, amsthm, mathtools}
\usepackage{xcolor}
\usepackage{subcaption}
\usepackage[frozencache]{minted}

\usepackage[compact]{titlesec}
\usepackage{xurl}
\usepackage{enumitem}
\setlist{nosep}

\iclrfinaltrue

\title{
The Art of Being Difficult: \\
Combining Human and AI Strengths to Find Adversarial Instances for Heuristics
}
\author{Henri Nikoleit \\ University of Bonn  \\ lumimail@proton.me
    \And Ankit Anand \\ Google DeepMind \\ anandank@google.com
    \And Anurag Murty Naredla \\ University of Manitoba \\ anurag.naredla@umanitoba.ca 
    \And Heiko Röglin   \\ University of Bonn \\ roeglin@cs.uni-bonn.de
}
\date{July 2025}

\newcommand{\changed}[1]{{#1}}

\newcommand{\methodname}{Co-FunSearch}
\newtheorem{theorem}{Theorem}[section]
\newtheorem{lemma}[theorem]{Lemma}

\newcommand{\weightprofit}[2]{\begin{pmatrix}{#1}\\{#2}\end{pmatrix}}
\newcommand{\Weight}{\operatorname{Weight}}
\newcommand{\Profit}{\operatorname{Profit}}

\begin{document}

\maketitle

\begin{abstract}

We demonstrate the power of human-LLM collaboration in tackling open problems in theoretical computer science. Focusing on combinatorial optimization, we refine outputs from the FunSearch algorithm [Romera-Paredes et al., Nature 2023] to derive state-of-the-art lower bounds for standard heuristics. Specifically, we target the generation of adversarial instances where these heuristics perform poorly. By iterating on FunSearch's outputs, we identify improved constructions for hierarchical $k$-median clustering, bin packing, the knapsack problem, and a generalization of Lovász's gasoline problem—some of these have not seen much improvement for over a decade, despite intermittent attention. These results illustrate how expert oversight can effectively extrapolate algorithmic insights from LLM-based evolutionary methods to break long-standing barriers.

Our findings demonstrate that while LLMs provide critical initial patterns, human expertise is essential for transforming these patterns into mathematically rigorous and insightful constructions. This work  highlights that LLMs are a strong collaborative tool in mathematics and computer science research.
\end{abstract}

\section{Introduction}

Artificial Intelligence has advanced mathematics and theoretical computer science significantly, driving progress by proposing new conjectures in knot theory~\citep{davies2021}, devising novel algorithms~\citep{fawzi2022discovering}, and discovering new lower bounds and combinatorial constructions~\citep{romera2024mathematical, novikov2025alphaevolve, wagner2021constructions, mehrabian2023finding}. While impactful, many of these efforts~\citep{fawzi2022discovering, wagner2021constructions, mehrabian2023finding} rely on black-box neural networks or conventional metaheuristics such as tabu search and simulated annealing. A drawback of these methods is their opacity; they often yield results without providing the structural insights necessary for human experts to generalize or understand the underlying mechanisms.

However, recently, large language models in works like \citep{romera2024mathematical, novikov2025alphaevolve} have addressed this problem by using programs to represent complex mathematical objects and solutions compactly. The generated program is interpretable by humans and can be potentially modified to guide the LLM iterations in the right direction. 

Many combinatorial optimization problems have widespread real-world applications but are computationally intractable (e.g., NP-hard). A natural way to address these problems in practice is to devise new heuristics. It is equally important to analyze and understand when and how these heuristics fail. Analyzing the worst-case performance of heuristics can explain their performance in real-world applications, and knowledge of adversarial instances can help devise better heuristics.

In this work, we use the Human-LLM collaboration to generate adversarial examples for heuristics for a variety of combinatorial optimization problems. Specifically, we target well-known algorithms for the knapsack problem, bin packing, hierarchical clustering, and a variant of Lovász's gasoline puzzle, establishing improved lower bounds for each. 
Our use of LLMs follows the FunSearch paradigm~\cite{romera2024mathematical} that has improved existing bounds in the cap-set problem and bin-packing heuristics. Unlike most previous work, human-AI collaboration for our target problems is necessary to provide the final theoretical results and proofs.

\begin{figure*}[htb]
\centering
    \includegraphics[width=\textwidth]{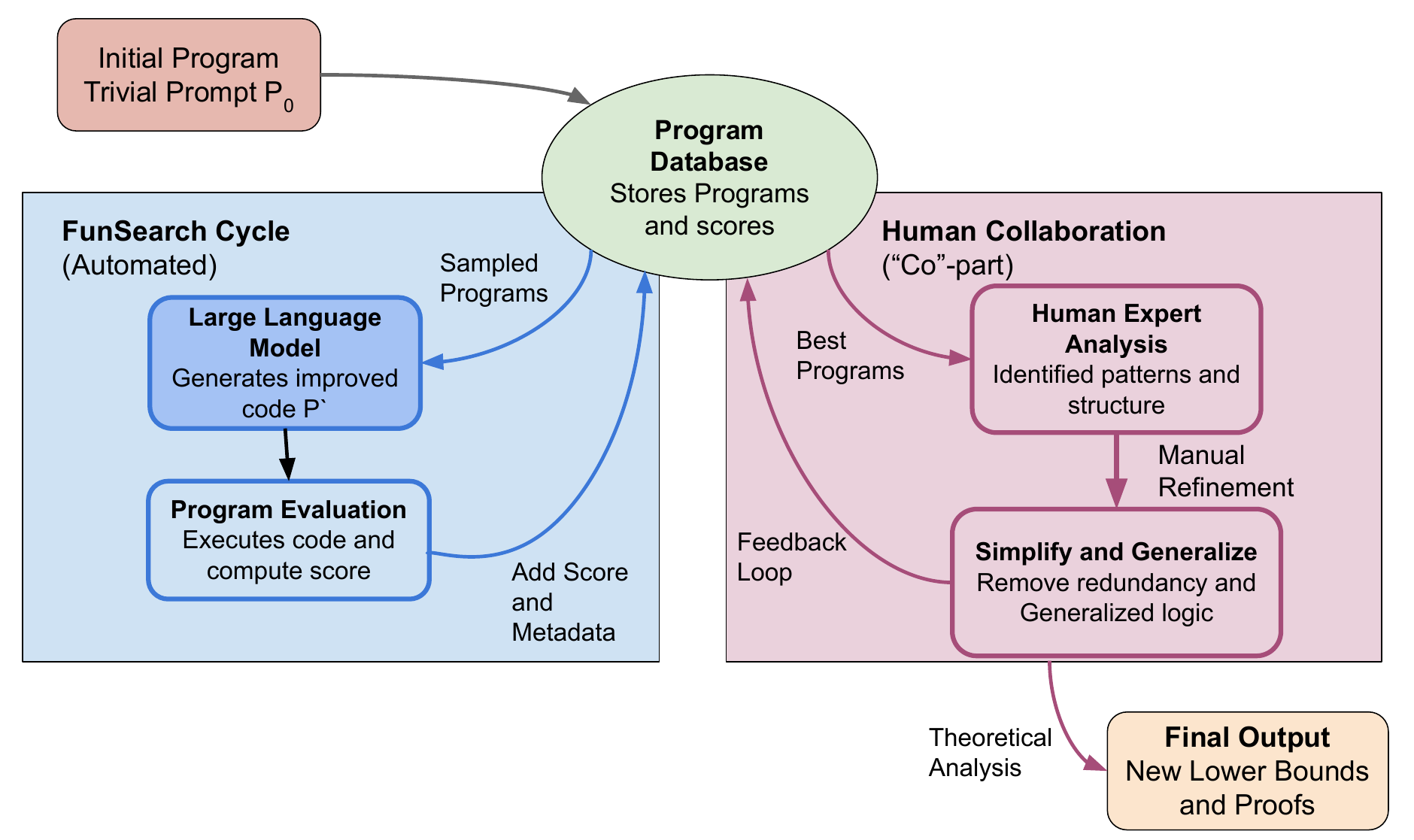}
    \caption{\changed{A diagrammatic representation of Co-FunSearch.}}
    \label{fig:diagram}
\end{figure*}

\noindent
\textbf{Methodology.} Our proposed framework, which we call, {\methodname} (short for {\em{Collaborative FunSearch}}), is summarized in Figure \ref{fig:diagram}.
 We begin with initial low scoring instances given to FunSearch as input and analyze the programs generated by FunSearch that achieve the highest scores. Some of the generated programs have interpretable structures relevant to the task, while others rely on hard-coded constants and offer little insight. We then manually refine the promising programs, removing components whose elimination does not reduce performance and simplifying the remaining logic wherever possible. For example, this may involve removing redundant elements of lists (Fig.~\ref{fig:code-bin-packing}, Fig.~\ref{fig:code-clustering}), or simplifying a list of $n$ ascending numbers into a list containing the mean of those numbers $n$ times (Fig.~\ref{fig:code-clustering}). Afterward, we attempt to prove statements about the scores of the instances, or otherwise feed the simplified programs back into FunSearch to obtain better results. These modifications were essential for generating meaningful structures, insights and obtaining state-of-the-art results.More importantly,  the collaborative workflow demonstrates FunSearch's potential for productive partnerships between computing experts and AI systems.

\noindent
\textbf{Summary of Results.} Table \ref{tab:main-results} summarizes our main results on all the problems. With {\methodname}, we were able to \emph{disprove} that the Nemhauser-Ullmann heuristic for the knapsack problem has output-polynomial running time, and we \changed{improve the lower bound} of the best fit heuristic for bin packing in the random order model \changed{from $1.3$ to $1.5$}. We also obtained the \emph{first non-trivial lower bound of the golden ratio} for the price of hierarchy for $k$-median clustering, and \emph{disprove the conjecture} that the iterative rounding algorithm for the generalized gasoline problem is a 2-approximation. We provide the source code for all the implementations at \url{https://github.com/lumi-a/funsearch}.

\noindent
\textbf{Why not Local Search?} FunSearch offers three distinct advantages over local search. First, while local search isolates single vectors, FunSearch discovers generic Python programs that scale with instance parameters. Second, it yields interpretable, modifiable code rather than opaque numeric vectors. Third, FunSearch exploits the low Kolmogorov complexity inherent in optimization problems, capturing structural symmetries that local search ignores. Notably, in bin packing, FunSearch generalized a pattern to \changed{obtain a lower bound of $1.5$ (surpassing the previous $1.3$)}, whereas local search produced unstructured solutions (see Table~\ref{tab:local-search-instance-comparison}). 

\begin{table*}[htb]
    \centering
    \begin{tabular}{lllll}
    \hline
    \hline
    Method & Knapsack & \changed{Bin-Packing} & $k$-median & Gasoline\\
    \hline\\
    Previous Best   \changed{Known}  Lower Bound & $2.0$ & $1.3$ & $1.0$ & $2.0$\\
    \hline
    Local Search & \changed{1.93} & $1.478$ & $1.36$ & $2.11$\\
    FunSearch & \changed{646.92} & $1.497$ & $1.538$ & $3.05$\\
    {\methodname} & $\mathbf{n^{O(\sqrt{n})}}$ & $\mathbf{1.5}$ & $\mathbf{1.618}$ & $\mathbf{4.65}$\\
    \hline
    Known  Upper Bound & $O(2^n)$ & \changed{1.7} & 16 &None \\
    \hline
    \end{tabular}
    \caption{Comparison of {\methodname} with base FunSearch, local search and SOTA on different problems. \changed{The given values for local search and FunSearch are the maxima across $30$ trials each.}}
    \label{tab:main-results}
\end{table*}

\begin{table}[t]
    \centering
    \begin{tabular}{llll}
        & Local Search & FunSearch & {\methodname}\\
        \hline
  Items & 0.003 & 0.08 & 0.167\\
        & 0.005 & 0.08 & 0.167\\
        & 0.006 & 0.08 & 0.167\\
        & 0.007 & 0.08 & 0.167\\
        & 0.021 & 0.08 & 0.167\\
        & 0.068 & 0.114 & 0.167\\
        & 0.073 & 0.114 & 0.143\\
        & 0.170 & 0.114 & 0.143\\
        & 0.202 & 0.114 & 0.143\\
        & 0.219 & 0.114 & 0.143\\
        & 0.306 & 0.114 & 0.143\\
        & 0.375 & 0.114 & 0.143\\
        & 0.540 & 0.2 & 0.143\\
        &       & 0.6  & \\
        \hline
    \end{tabular}
    
    \caption{Comparing the final instances found by local search, FunSearch and {\methodname} for the randomised Best-Fit bin-packing problem.}
    \label{tab:local-search-instance-comparison}
\end{table}

\newcommand{\Score}{\operatorname{Score}}

\section{Problems and Notation}
\subsection{General Framework for Adversarial Instance Generation}
We first propose a general framework for generating adversarial instances for any given heuristic, and then describe the particular problems we focus on in this work and how we instantiate this general framework for the given problem. Given an optimization problem (without loss of generality, a minimization problem), a heuristic algorithm $\mathcal{H}$ and a (computationally expensive) optimal algorithm {\it{Opt}}, the goal is to construct an instance $\mathcal{I}$ where the heuristic performs poorly with respect to {\it{Opt}}. More concretely for minimization problems, we aim to construct an adversarial instance $\mathcal{I}$ such that the ratio $R = \frac{\Score(\mathcal{H}(\mathcal I))}{\Score(\it{Opt}(\mathcal I))}$ is large, where  $\Score(\mathcal{H}(\mathcal I))$ denotes the value yielded by the heuristic algorithm and $\Score(\it{Opt}(\mathcal I))$ denotes the optimum value for $\mathcal I$.

While methods like local search, tabu search, and genetic algorithms have focused on generating adversarial instances for heuristics, this work focuses on using language models for generating the instances. Specifically, we model each instance as output of a program $\mathcal{P}$ s.t.\ $\mathcal{I} = \operatorname{Output}(\mathcal{P})$. Initially, a trivial instance is expressed as program $\mathcal{P}_0$. In addition, we prompt a large language model $\mathcal{L}$ that has proficiency in code generation and reasoning. At each iteration $i$, the language model takes as input one of the previously generated programs, $ p= \mathcal{P}_{<i}$ and generates an improved version $p'$ of $p$ such that it improves the reward $R$. We specifically follow the evolutionary approach used in \citet{romera2024mathematical} for generating these programs and optimizing the reward $R$.

\subsection{Problems and Heuristics}
We focus on four distinct problems and their corresponding heuristics to illustrate the effectiveness of this approach. These problems vary from knapsack, bin-packing, hierarchical clustering to the gasoline puzzle by Lov\'asz. While the approach is general, we believe the specific instantiation on these problems provides a general lens to find adversarial instances for any given heuristic.

\subsubsection{Nemhauser-Ullmann heuristic for the knapsack problem} \label{sec:knapsack}
In the classical NP-hard knapsack problem, an input consists of a set of $n$ items, where each item $i\in[n]$ has a profit $p_i\in\mathbb{R}_{>0}$ and a weight $w_i\in\mathbb{R}_{>0}$. Additionally, a capacity $t\in\mathbb{R}_{>0}$ is given, and the goal is to find a subset $I\subseteq[n]$ of the items such that the profit $\sum_{i\in I}p_i$ is maximized under the constraint $\sum_{i\in I}w_i \le t$. Without a given capacity $t$, the knapsack problem can also be viewed as a bi-objective optimization problem, where one wants to find a subset with small weight and large profit. These two objectives are obviously conflicting and there is no clear optimal solution anymore, but one rather has to find a good trade-off between the criteria. In multi-objective optimization, it is very common to compute the set of Pareto-optimal solutions where a solution is called Pareto-optimal if there does not exist another solution that is simultaneously better in all objectives (see, e.g., \cite{Ehrgott2005} for a comprehensive overview). Only Pareto-optimal solutions constitute reasonable trade-offs and for many multi-objective optimization problems, algorithms for computing the set of Pareto-optimal solutions are known  (e.g., for the multi-objective shortest path problem~\citep{CorleyM85}). These are usually no polynomial-time algorithms, as the set of Pareto-optimal solutions can be of exponential size. However, in practice the Pareto set is often small and one is interested in finding algorithms that are output-polynomial time, i.e., whose running time depends polynomially on the input and the output size. Such algorithms are efficient if the Pareto set is small, which is often the case in applications.

\paragraph{Nemhauser-Ullmann Heuristic} It is an open problem whether output-polynomial time algorithms for the knapsack problem (viewed as a bi-objective optimization problem) exist~\citep{RoeglinBookChapter}. The best candidate for such an algorithm is the Nemhauser-Ullmann algorithm, which is based on dynamic programming~\citep{NU69}. For a given instance of the knapsack problem with $n$ items, it computes iteratively the Pareto sets $\mathcal{P}_1,\ldots,\mathcal{P}_n$, where $\mathcal{P}_i$ denotes the Pareto set of the sub-instance that consists only of the first $i$ items (i.e., $\mathcal{P}_n$ is the Pareto set of the entire instance). The Nemhauser-Ullmann algorithm can be implemented to run in time $O(\sum_{i=1}^{n} |\mathcal{P}_i|)$. If there was an $\alpha$ such that  $|\mathcal{P}_i|\le\alpha|\mathcal{P}_n|$ for each instance and each $i$, one could bound the running time by $O(\alpha n |\mathcal{P}_n|)$, which would result in an output-polynomial time algorithm as long as $\alpha$ grows at most polynomially with $n$. So far, no instances were known where an intermediate set $\mathcal{P}_i$ is larger than the final Pareto set $\mathcal{P}_n$ by more than a small constant factor. With the help of an instance generated by FunSearch, we construct a sequence of instances disproving that the Nemhauser-Ullmann algorithm has output-polynomial running time.

\subsubsection{Best-Fit Heuristic for Bin Packing}\label{section:bin}

 Bin Packing is a classical NP-hard optimization problem that has been studied extensively as an online problem. In this problem, items with sizes $w_1,w_2,w_3,\ldots$ arrive one by one and an online algorithm has to assign each item irrevocably to a bin when it arrives. There is an unlimited number of bins with a fixed capacity $c$ available. The goal is to use as few bins as possible to pack all items. In the online setting, simple  algorithms like First-Fit and Best-Fit have been studied, which pack each arriving item into the first bin into which it fits or the fullest bin into which it fits, respectively. To mitigate the power of the adversary in classical worst-case analysis, these algorithms have been studied extensively in the random order setting, in which an adversary chooses the items' sizes but the items arrive in a random order. \changed{In the unshuffled setting, \citet{bestFitAbsoluteRatio} proved an upper bound of $1.7$ on the approximation-ratio of Best-Fit. This means that, on any instance, the expected number of bins used by Best-Fit is at most $1.7$ times the optimal number. As this holds for any instance, this upper bound also applies to the shuffled setting. In the shuffled setting, the best-known lower bound was $1.3$, i.e., there exists an instance such that, when the instance is shuffled, Best-Fit needs at least $1.3$ times the optimal number of bins, in expectation \citep{binPackingRevisited}. With the help of FunSearch, we improve this lower bound to $1.5$.}

\subsubsection{k-median in Hierarchical Clustering} \label{sec:hierarchical-clustering}

Hierarchical clustering is an important research topic in unsupervised learning. In such a clustering problem, usually a data set $X$ with $n$ points is given and one seeks for a sequence $\mathcal{H}_1,\ldots,\mathcal{H}_n$ of clusterings, where each $\mathcal{H}_k$ is a $k$-clustering of $X$, i.e., a partition of $X$ into $k$ parts. The clusterings must be hierarchically compatible, meaning that each $\mathcal{H}_k$ is obtained from $\mathcal{H}_{k+1}$ by merging two clusters. To evaluate the quality of such a hierarchical clustering, a common approach is to choose an objective function $\Phi$ like $k$-center, $k$-median, or $k$-means and to compare each clustering $\mathcal{H}_k$ with an optimal $k$-clustering $\mathrm{OPT}_k$ with respect to the objective $\Phi$. Then the approximation factor $\alpha$ of the hierarchical clustering can be defined as the worst approximation factor of any of the levels, i.e., $\alpha=\max_{k\in[n]} \Phi(\mathcal{H}_k)/\Phi(\mathrm{OPT}_k)$ (see, e.g., \cite{LinEtAl2010}). Since the optimal clusterings are usually not hierarchically compatible, an approximation factor of $1$ cannot be achieved even with unlimited running time. ~\citet{priceOfHierarchicalClustering} defined the \emph{price of hierarchy} of a clustering objective $\Phi$ as the best approximation factor that can be achieved for any clustering instance. They showed, e.g., that the price of hierarchy for the $k$-center objective is exactly $4$, meaning that for any instance of the hierarchical $k$-center problem there exists a hierarchical clustering with an approximation factor of $4$ and that there exists an instance for which any hierarchical clustering does not have a better approximation factor than $4$. For the $k$-median problem, no non-trivial lower bound on the price of hierarchy is known. The best known upper bound is 16 for general metrics \citep{dai2014}. We obtain the first non-trivial lower bound for the price of hierarchy for the $k$-median problem, showing that it is at least the golden ratio, ${\approx}1.618$.

\subsubsection{Gasoline Problem}\label{sec:gasoline}

The Gasoline problem is a combinatorial optimization problem inspired by Lov\'asz's gasoline puzzle~\citep{lovasz2007combinatorial}. In an instance of this problem, we are given two sets $X=\{x_1,\ldots,x_n\}$ and $Y=\{y_1,\ldots,y_n\}$ of non-negative numbers with the same sum. The goal is to find a permutation $\pi$ of the set $X$ that minimizes the value of $\eta$ such that
{\small\[
\forall [k,\ell]: \quad \Biggl|\sum_{i\in[k,\ell]} x_{\pi(i)} -
\sum_{i\in[k,\ell-1]}y_{i}\Biggr| \leq \eta.
\]}%
Given a circle with $n$ points labeled $1$ through $n$, the interval
$[k,\ell]$ denotes a consecutive subset of integers assigned to points
$k$ through $\ell$.  For example, $[5,8] = \{5,6,7,8\}$, and $[n-1,3] = \{n-1,n,1,2,3\}$. The intuition is that the $y_i$-values correspond to road segments on a cycle and the $x_i$-values correspond to fuel canisters that can be placed between the segments. The goal is to distribute the canisters such that one can get around the cycle with the smallest possible fuel tank capacity $\eta$.

The Gasoline problem is known to be NP-hard, and a 2-approximation algorithm for it is known \citep{Gasoline2018}. It is an open problem whether better approximation algorithms or even a polynomial-time approximation scheme exist. In the literature, another heuristic for the problem has been considered that is based on iteratively rounding the linear programming relaxation \citep{rajkovic}. The approximation guarantee of this algorithm is unknown. In his master's thesis, Lorieau constructed a class of instances showing that its approximation factor is not better than $2$ \citep{Lorieau}. Lorieau conjectured that it is actually a $2$-approximation algorithm, but this has not been proven yet. The iterative rounding algorithm is interesting because it generalizes canonically to a $d$-dimensional Gasoline problem in which $x_i$ and $y_i$ are $d$-dimensional vectors. Also for this generalization, the best-known lower bound was $2$ and Lorieau conjectured that also for this generalization the algorithm achieves a $2$-approximation. With {\methodname}, we obtain a family of instances disproving this conjecture.

\section{Experimental Details and Results}\label{sec:experimental-details}

We compare {\methodname} to base FunSearch and local search on the above 4 problems. The main goal in all these problems is to search for a vector $v$ \changed{(encoding the instance)} which optimizes the given objective \changed{(usually some performance-measure of some heuristic on this specific instance). The objectives depend on the problem, and can be found in Section \ref{sec:key-results}.}
Random search works by initializing a random vector $v$. At each step, sample a random vector $v'$ close to $v$ and check if $v'$ improves on the objective. If it does, replace $v$ by $v'$ with some probability $p$, otherwise keep $v$ unchanged. This procedure keeps improving on the objective until reaching a local minimum. For our experiments, $v'$ arises from $v$ by adding independent normally-distributed noise with mean $0$ and variance $s \cdot (1-\frac{t}{t_{\max}})$ to each coordinate of $v$ (clipping $v'$ to the problem's bounds as required), where $t$ is the current time since the start of the search, $t_{\max}$ is the time after which we terminate the search (set to 3 minutes), and $s$ is a problem-specific parameter. For the knapsack-problem, we chose $20$ items and $s=1000$, because both weights and profits were rounded before evaluation to be less sensitive to floating-point imprecision. For bin-packing, we chose $13$ bins with capacity $1$ and $s=0.2$. For weighted hierarchical clustering, we chose $8$ points, $s=0.2$, and replaced each point's weight $w$ to $2^w$ before evaluation, because we observed worst-case instances' weights frequently spanning several orders of magnitude. For the two-dimensional gasoline-problem, we chose $s=0.2$ and $|X|=|Y|=14$.

FunSearch works similarly: Instead of searching for a vector $v$ that has a high objective, it searches for a Python-program $P$ outputting a vector with high objective. Sampling a Python-program $P'$ ``close" to $P$ is not done by randomly changing characters in the program's source-code, but by prompting an LLM with the source-code of $P$, requesting a similar program which improves the score. The scoring-function is not provided to the LLM. The newly generated program (if it executes without error) is added to a database of programs with its score. In the next iteration, a new program is sampled from the database according to a probability distribution and the process is repeated. More details about the evolutionary search can be found in \citet{romera2024mathematical}. To evaluate a given program, we use problem-specific scoring-functions, described in their respective sections below.

\subsection{Results}

Table \ref{tab:main-results} outlines the main results for all four problems. Our main results are as follows:

\begin{itemize}
    \item For the knapsack problem, the local search method \changed{only achieves $1.93$, whereas FunSearch found instances with a score of $646.92$.} The compact program found by FunSearch could be improved to a general super-polynomial bound $n^{O(\sqrt{n})}$. Since then, an unrelated exponential bound has been found, discussed in \citep{mastersthesis}. Also this bound was optimized using Co-FunSearch.
    \item For the Best-Fit heuristic for bin packing, FunSearch finds an instance which is $1.497$ times worse than optimal, outperforming both the existing SOTA ($1.3$) and local search ($1.478$). This instance could easily be generalized, yielding an asymptotic bound of $1.5$.
    \item For the hierarchical $k$-median problem, no non-trivial lower bounds were previously known. FunSearch ($1.538$) outperforms local search ($1.36$) with an instance that we could modify to yield a lower bound of the golden ratio (${\approx}1.618$). 
    \item Lastly, in Lov\'asz's gasoline problem, FunSearch ($3.05$) outperforms both the SOTA ($2.0$) and local search ($2.11$), and could be further improved to $4.65$.
\end{itemize}

\paragraph{Generated Programs with FunSearch and {\methodname}}
In this section, we illustrate the programs found by FunSearch and how these programs are modified by experts to obtain adversarial instances which are much better in score and are generalizable with guarantees. Fig.~\ref{code:bin-packing-start} shows the initial program given in the bin-packing problem, Fig.~\ref{code:bin-packing-funsearch-output} shows the instance generated by Fun-Search, which achieves a score of $1.4978$, and Fig.~\ref{code:bin-packing-co-funsearch} shows how we generalized this instance: The instance consists of two types of items in a list which are generalized as entries ``$a$" and ``$b$" in the figure. Specifically, for large $a$ and $b$, this instance's score approaches $1.5$. Similar to Figure~\ref{fig:code-bin-packing}, we compare the initial program, the program generated by FunSearch, and the program obtained via human collaboration for the knapsack      (Figure \ref{fig:code-nemhauser-ullmann}), hierarchical clustering (Figure \ref{fig:code-clustering}), and the gasoline problem (Figure \ref{fig:code-gasoline}).

\begin{figure*}[htb]
    \centering
    \begin{subfigure}[b]{\textwidth}
        \centering
        \begin{minted}[breaklines,fontsize=\scriptsize]{python}
def get_items() -> list[float]:
    """Return a new bin-packing-instance, specified by the list of items.

    The items must be floats between 0 and 1."""
    items = [0.4, 0.5, 0.6]
    return items
        \end{minted}
        \caption{Initial program.}
        \label{code:bin-packing-start}
    \end{subfigure}
      \begin{subfigure}[t]{\textwidth}
        \centering
        \begin{minted}[breaklines,fontsize=\scriptsize]{python}
def get_items() -> list[float]:
    """Return a new bin-packing-instance, specified by the list of items.

    The items must be floats between 0 and 1."""
    """Yet another version of `get_items_v0`, `get_items_v1`, and `get_items_v2`, with some lines altered."""
    items = [0.8, 0.2, 0.6, 0.4]
    # Split the first item into seven smaller items and the fourth item into five smaller items
    items = [0.114, 0.114, 0.114, 0.114, 0.114, 0.114, 0.114] + items[1:3] + [0.08, 0.08, 0.08, 0.08, 0.08]
    return items
        \end{minted}
        \caption{A program found by FunSearch after $10$ trials of 2,400 samples each.}\label{code:bin-packing-funsearch-output}
    \end{subfigure}
    
    \begin{subfigure}[b]{\textwidth}
        \centering
        \begin{minted}[breaklines,fontsize=\scriptsize]{python}
def get_items() -> list[float]:
    a = 7
    b = 5
    return [1.0 / a] * a + [1.0 / b] * b
        \end{minted}
        \caption{An intermediate step of tuning program in Figure \ref{code:bin-packing-funsearch-output} by hand}
        \label{code:bin-packing-co-funsearch}
    \end{subfigure}
    \\[0.5em]

    \caption{The evolution of programs generating bin packing instances, with model open-mistral-nemo and a temperature of $1.5$.}
    \label{fig:code-bin-packing}
\end{figure*}

\subsection{Ablations}
Figure \ref{fig:ablations} shows the search dynamics with variations across different models, the temperature parameter and the initial program used during FunSearch. In all these experiments, we plot the maximum score of samples produced so far against the number of samples (LLM-prompts), together with the standard error across $30$ trials. To illustrate the effect of variations and due to high computational cost (inference costs) of each experiment, we undertake these ablations on a single problem but believe similar trends would hold for all the other problems as well.

\textbf{Variations across different models:} Fig.~\ref{fig:models} shows the variations with two models from \mbox{OpenAI}, gpt-4.1-mini \citep{gptmini} and gpt-4.1-nano \citep{gptnano} with Mistral AI's open-mistral-nemo model \citep{openmistralnemo}. We observe that gpt-4.1-nano slightly outperforms gpt-4.1-mini. This is a bit counterintuitive, as gpt-4.1-mini is a more powerful model than gpt-4.1-nano. To investigate this further, we plot the both the maximum score and the rolling average score of the last 10 samples (Figure \ref{fig:nano-vs-mini}). Here, gpt-4.1-mini outperforms gpt-4.1-nano on the rolling average but performs slightly poorer on the maximum score, highlighting that, although larger models are stronger on average, in problems with verifiable score where one cares about the best performing sample, smaller models are sufficient and can outperform larger ones.

\textbf{Variations across temperature:} Fig.~\ref{fig:temperature} shows the variations of the objective function with the change in the sampling temperature. The sampling temperature is an indicator of sharpness of the LLM's probability distribution for each sample (the lower the temperature, the more sharp it is).  We observe that the higher sampling temperature performs better than lower sampling temperature, owing to high entropy of samples produced in inference. It should be noted that we plot the best score obtained across all samples as objective, so even if the mean performance drops, the best sample is better owing to increase in entropy and diversity.

\textbf{Variations across initial prompts:} Another critical hyperparameter in FunSearch as outlined by \citet{romera2024mathematical} is the initial instance given to a FunSearch experiment. In Fig.~\ref{fig:startpoints}, we vary the initial program for FunSearch on the bin packing problem. We observe that a trivial instance with a more flexible structure (a for-loop adding the items $1/i$ for $i\in\{1,...,10\}$) starts from a low initial score but improves as more and more samples are drawn in FunSearch. Additionally, we hard-code a trivial instance as numbers without appropriate structure, and although this improves with more samples, the performance is inferior to both the trivial instance with the structure and the best known construction. Observing the output, the variations introduced by FunSearch consist of different hardcoded numbers, as opposed to inserting more structure, like loops or maths-functions, into the program. This highlights the importance of an appropriate structure and skeleton for the initial program in FunSearch. We compare this with using the best known construction \citep{binPackingRevisited} as the initial instance, which does start from a high score initially but stagnates quickly with iterations.

\begin{figure*}[]
\centering
    \begin{subfigure}{0.45\textwidth}
        \includegraphics[width=\textwidth]{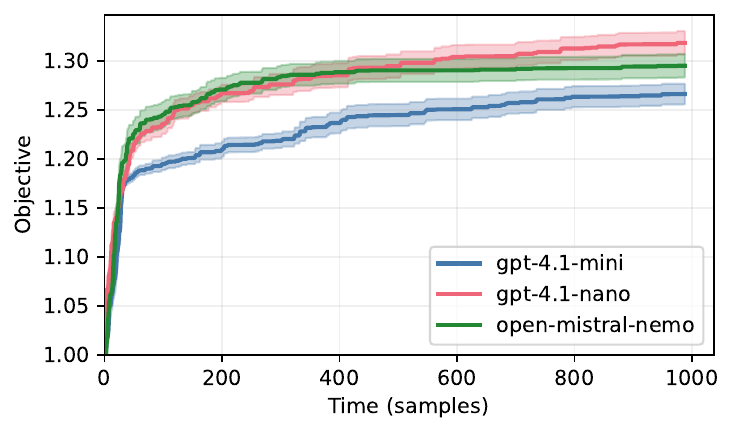}
        \caption{}
        \label{fig:models}
    \end{subfigure}
    \begin{subfigure}{0.45\textwidth}
        \includegraphics[width=\textwidth]{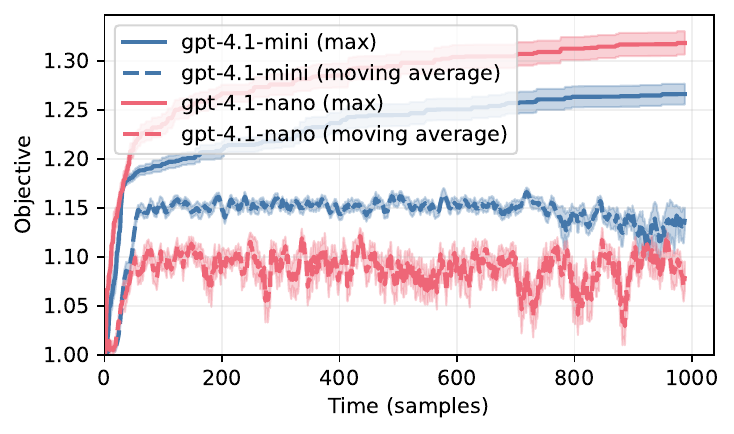}
        \caption{}
        \label{fig:nano-vs-mini}
    \end{subfigure}
    \begin{subfigure}{0.45\textwidth}
        \includegraphics[width=\textwidth]{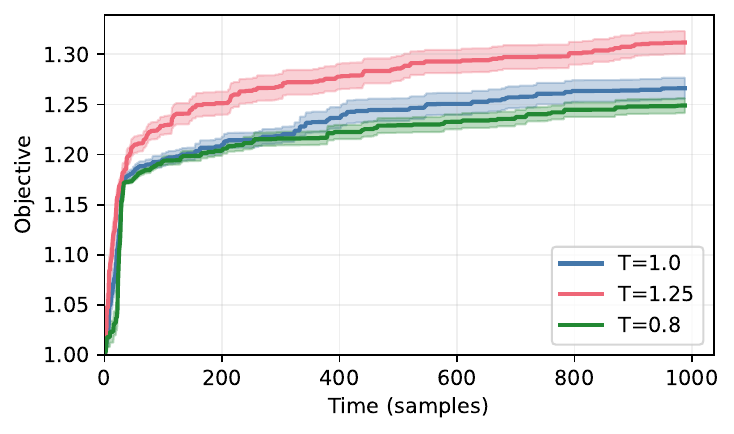} 
        \caption{}
        \label{fig:temperature}
    \end{subfigure}
    \begin{subfigure}{0.45\textwidth}
        \includegraphics[width=\textwidth]{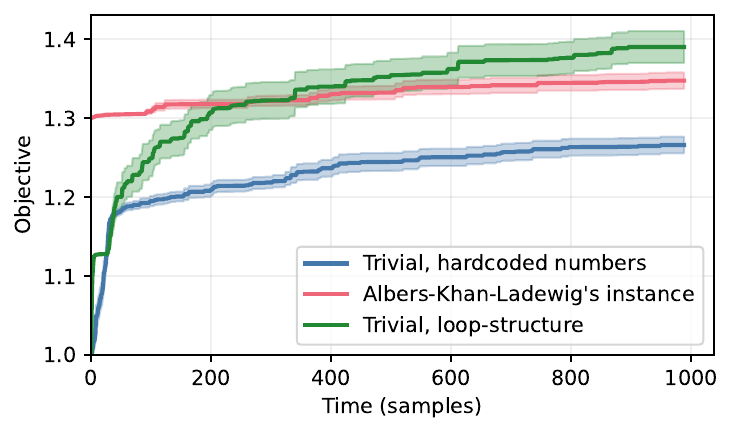}
        \caption{}
        \label{fig:startpoints}
    \end{subfigure}
    \caption{Comparing the effect of different hyperparameters on the objective function in bin packing.(a) Comparing different models, each with temperature 1.0 and starting with a hard-coded instance.(b) Comparing rolling average (10 samples) and max-performance of gpt-4.1-mini with gpt-4.1-nano, with temp: $1.0$.(c) Variation of different sampling temperatures for gpt-4.1-mini, each starting with a hard-coded instance.(d) Variation of initial instances for gpt-4.1-mini with temperature $1.0$.}
    \label{fig:ablations}
\end{figure*}

\subsection{{\methodname} and Key Results} \label{sec:key-results}
In this section, we highlight how we used FunSearch to find instances and generalized them to achieve improved lower bounds for each problem. Furthermore, we also provide proofs for lower bounds for most of these instances.

\subsubsection{Knapsack Problem} \label{sec:knapsack-res-instances}
We consider the knapsack problem (as described in Section \ref{sec:knapsack}) as a bi-criteria optimization problem, where we want to minimize the total weight while maximizing the total profit. We used FunSearch to find instances $I$ that have a high score $\max_{1\leq i\leq n} |{\mathcal P}_i(I)| /|{\mathcal P}(I)|$, i.e., where the Pareto set ${\mathcal P}_i(I)$ of a sub-instance $I_{i}$, which consists only of the first $i$ items of $I$, is much larger than the Pareto set ${\mathcal P}(I)$ of the entire instance $I$. The sizes of the intermediate and final Pareto-sets are obtained as a by-product of running the Nemhauser-Ullmann algorithm on $I$. Items are written as tuples of the form (weight, profit).

We obtain the code (as shown in Figure~\ref{code:knapsack-funsearch-output} in Appendix) after running FunSearch for 10 trials of 500 samples each.
Having simplified the output (shown in Fig.~\ref{code:knapsack-co-funsearch}), we can scale all items' weights up by a factor of $2$ (which does not affect Pareto-optimality), decrease some profits by $1$, and change the last item to obtain the following tidier instance, which achieves slightly higher scores for the same $n$:
\small{
\[
    \biggl[
        \underbrace{\weightprofit{8}{8},...,\weightprofit{8}{8}}_{n \text{ times}},\ \ 
        \underbrace{\weightprofit{2}{1},...,\weightprofit{2}{1}}_{n \text{ times}},\ \
        \weightprofit{4}{4},\weightprofit{2}{2}
    \biggr].
\]}%
From here, we attempted to prove results about the instance. After a first draft, we found it more natural to replace the first $n$ items by $n$ powers of $2$, and saw that stronger results are possible by replacing the last two items by $k$ powers of $2$:
{\small{
\begin{alignat*}{1}
    \biggl[
        \weightprofit{2^{2k}}{2^{2k}},\weightprofit{2^{2k+1}}{2^{2k+1}},...,\weightprofit{2^{2k+n}}{2^{2k+n}}&,\ \ 
        \underbrace{\weightprofit{2^k}{2^k - 1},...,\weightprofit{2^k}{2^k -1}}_{n \text{ times}},\\
        \weightprofit{2^{2k-1}}{2^{2k-1}},\weightprofit{2^{2k-2}}{2^{2k-2}},&...,\weightprofit{2^{k+1}}{2^{k+1}}
    \biggr].
\end{alignat*}}}%
Finally, to apply our result not only to the size of the Pareto sets but also to the runtime of the Nemhauser-Ullmann algorithm\footnote{The difference between the size of the Pareto set and the running time of the Nemhauser-Ullmann algorithm is that, for the Nemhauser-Ullmann algorithm, multiple Pareto-optimal solutions with exactly the same profit and weight are treated as a single solution for the running time.}, we appended the factors $x_i \coloneq (1+\frac{2^{-i}}{2^k-1})$ to the $n$ center items:
{\small{
\begin{equation}\label{eq:knapsack-instance}
    \begin{aligned}
    \biggl[
        \weightprofit{2^{2k}}{2^{2k}},...,\weightprofit{2^{2k+n}}{2^{2k+n}},\ \ 
        \weightprofit{x_1 \cdot 2^k}{x_1\cdot(2^k - 1)},...,\\\weightprofit{x_n \cdot 2^k}{x_n\cdot(2^k -1)},\
        \weightprofit{2^{2k-1}}{2^{2k-1}},...,\weightprofit{2^{k+1}}{2^{k+1}}
    \biggr].
    \end{aligned}
\end{equation}}}%
By choosing $k = \log_2(\sqrt n)+1$, this instance shows:

\begin{theorem}
    The Nemhauser-Ullmann algorithm does not have output-polynomial running time.
\end{theorem}

 Before this work, no such instances were known. We refer to Appendix~\ref{appendix:knapsack-res-instances} for further details and proofs. After finding this instance, we found an independent construction that even shows an exponential lower bound. See \cite[Corollaries 4.2.10 and 4.2.11]{mastersthesis} for details. This lower bound was also obtained by Co-FunSearch.

\subsubsection{Bin-Packing} \label{sec:binpacking-res-instance}
As outlined in Section~\ref{section:bin}, we study the Best-Fit heuristic for the bin packing problem in the random-order setting. To evaluate a generated instance, we compute the value $v_{\operatorname{opt}}$ of an optimum solution with a solver described and implemented in \cite{fontan}, then compute the mean $v_{\operatorname{appx}}$ of 10,000 trials of the Best-Fit algorithm when the instance is shuffled randomly, and assign the instance a score of $\frac{v_{\operatorname{appx}}}{v_{\operatorname{opt}}}$. Fig.~\ref{fig:code-bin-packing} shows the programs generated by FunSearch. It is straightforward to observe that Fig.~\ref{code:bin-packing-funsearch-output} has multiple repetitions. We simplified this code to a list with only two parameters (Fig.~\ref{code:bin-packing-co-funsearch}).

\textbf{Instance Generated by {\methodname}:} For fixed $m \in \mathbb N$, consider the instance with maximum bin capacity $c \coloneq m\cdot(m+1)$ and items:
\[
    [\underbrace{m+1, \dots, m+1}_{m\text{ times}}, \underbrace{m, \dots, m}_{m+1\text{ times}}].
\]
An optimal packing puts the first $m$ items into one bin, and the remaining $m+1$ items into a second bin. This fills both bins exactly to their maximum capacity. Because $m$ and $m+1$ are coprime, these are the only two ways of filling a bin exactly to its maximum capacity $c$. Hence, if any bin contains both an item $m$ and an item $m+1$, the packing must use at least $3$ bins. Because the instance is shuffled, Best Fit will put both an item of size $m$ and an item of size $m+1$ into the same bin with high probability, approaching probability $1$ for growing $m$. Thus, with high probability, Best-Fit will use at least $3$ bins, which shows that the absolute random-order ratio of Best-Fit is at least $3/2$ (the previous best known lower bound was $1.3$, by \citet{binPackingRevisited}). Combining with the results of \changed{\citet{bestFitAbsoluteRatio}}, we obtain the following theorem:

\begin{theorem}
    The absolute random-order ratio of Best-Fit is \changed{between $1.5$ and $1.7$}.
\end{theorem}

\subsubsection{Hierarchical Clustering} \label{sec:clustering-res-instance}
In clustering, we're given a set of $n$ weighted points and a number $k$, with the task of finding a partition of the set of points into $k$ \emph{clusters} such that the total cost of the clustering is small. In $k$-median clustering, the points are a finite subset of $\mathbb{R}^d$ and the cost of a cluster $C$ is defined as the sum of the weighted $L^1$-distances all points have to the center, where the center is the best possible choice within that cluster:
\[
    \operatorname{Cost}(C)
    = \min_{p\in C} \sum_{q \in C} w(q)\|p-q\|_1
\]
Here, $w(q)$ is the weight of $q$ as specified by the instance. The total cost of a clustering is the sum of the costs of its clusters.

Clustering is used to analyze empirical data, but it's usually not clear what number of clusters $k$ is a good choice for the dataset. Instead of computing a clustering for a fixed $k$, one could compute a \emph{Hierarchical Clustering}, which has a clustering for each $k \in \{1,...,n\}$ and these clusterings are nested: A hierarchical clustering $H = (H_1,...,H_n)$ consists of $n$ clusterings such that, for all $i\in\{2,...,n\}$, $H_i$ is obtained by merging two clusters of $H_{i+1}$.

While hierarchical clusterings have an intuitive structure and don't require to decide on a number $k$ of clusters beforehand, they come at the disadvantage of their clusterings $H_i$ possibly having a higher cost than the \emph{optimal} $i$-clustering, because optimal clusterings need not form a hierarchy. For a given instance (a finite set of points in $\mathbb{R}^d$) $I$, we can measure the performance of a hierarchical clustering $H$ by comparing each of its clusterings $H_i$ to the best $i$-clustering, and choosing the level where this ratio is highest.

To measure how much we sacrifice when restricting ourselves to hierarchical clusterings on an instance $I$, we consider the \emph{Price of Hierarchy of $I$} as the best hierarchical clustering according to that measure:
\[
    \operatorname{PoH}(I)
    \coloneq
    \min_{H}
    \max_{k\in\{1,...,n\}}
    \Bigl[
    \frac{\operatorname{Cost}(H_k)}{\operatorname{Cost}(\operatorname{OPT}_k)}
    \Bigr],
\]
where $\operatorname{OPT}_k$ denotes an optimal $k$-clustering for $I$.

The \emph{Price of Hierarchy for $k$-median clustering $\operatorname{PoH}_{k\text{-median}}$} denotes the worst-case Price of Hierarchy of $I$ across all possible instances $I$. Thus, $\operatorname{PoH}_{k\text{-median}}$ captures the worst-case quality of an optimal hierarchical clustering when compared to an optimal non-hierarchical clustering.

With Co-FunSearch, we found the following lower bound construction for the price of hierarchy for $k$-median clustering (see also Fig.~\ref{fig:code-clustering}).
Fix the dimension $d\geq4$. Put $c \coloneq \frac{\sqrt{4d^2 + (3-d)^2} + d - 3}{2}$, which is one of the two roots of $0 = c^2 - c(d-3)-d^2$. Because $d\geq4$, we know that $5d^2 - 6d \geq 4d^2$, hence:
\[
    c
    = \frac{\sqrt{4d^2 + (d-3)^2} + d - 3}{2}
    > \frac{2d + d - 3}{2}
    > d.
\]
Let $e_i$ be the $i$th $d$-dimensional standard basis vector. Consider the following weighted instance of $d+2$ points:
\[
    (1,\dots,1),\quad
    (0,\dots,0),\quad
    -c e_1,\ 
    \dots,\ 
    -c e_d,
\]
where the point $(1,\dots,1)$ has weight $\infty$ and all other points have weight $1$.

\begin{theorem}\label{theorem:clustering}
For $k$-median clustering, this instance's price of hierarchy is at least $\frac{c}{d}$.
\end{theorem}

\begin{proof}
For contradiction, assume there exists a hierarchical clustering $H = (H_1,\dots,H_{d+2})$ such that, on every level, the cost of $H_k$ is strictly less than $\frac{c}{d}$ times the cost of the best clustering using $k$ clusters. This enables us to narrow down the structure of $H$:
\begin{itemize}
    \item For $k=d+1$, there is one cluster $C$ containing two points, while all other clusters contain only a single point. Depending on which two points constitute $C$, we can calculate the total cost of the clustering:
    \begin{itemize}
        \item If $C = \{(0,\dots,0), (1,\dots,1)\}$, the total cost is: \[\lVert (0,\dots,0)- (1,\dots,1)\rVert_1 = d.\]
        \item If $C = \{(0,\dots,0), -c e_i\}$ for some $i$, the total cost is $c$.
        \item If $C = \{(1,\dots,1), -ce_i\}$ for some $i$, the total cost is $d+c$.
        \item If $C = \{-ce_i, -ce_j\}$ for some $i\neq j$, the total cost is $2c$.
    \end{itemize}
    Because $d < c$, this constrains $H_k$ to $C = \{(0,\dots,0), (1,\dots,1)\}$, otherwise the total cost of $H_k$ would be at least $\frac{c}{d}$ times the cost of an optimal $(d+1)$-clustering.
    \item For $k=2$: The clustering now contains exactly two clusters. Because $H$ is a hierarchical clustering, we now know that $H_2$ has a cluster that contains $(0,\dots,0)$, $(1,\dots,1)$ and some number $0\leq n\leq d-1$ of the $- ce_i$, while its other cluster contains the remaining $d-1-n$ of the $-c e_i$. Due to symmetry, this number $n$ is sufficient for calculating the total cost of $H_2$. Because $(1,\dots,1)$ has infinite weight, this point must be the center of the first cluster, so this cluster has cost:
    \[
        \bigl\lVert (1,\dots,1) - (0,\dots,0)\bigr\rVert_1 + n\cdot \bigl\lVert (1,\dots,1) - (-c e_1)\bigr\rVert_1
        = d + n\cdot (c+d)
    \]
    The cluster containing the remaining $d-1-n$ of the $-c e_i$ can choose any point as its center. It has cost:
    \[
        (d-2-n)\cdot \bigl\lVert c e_1 - ce_2\bigr\rVert_1
        = (d-2-n)\cdot 2c
    \]
    Given $n$, the total cost of $H_2$ is $d+c(2d-4) + n(d-c)$. Because $d-c < 0$, the best choice for $n$ would be $n=d-1$, resulting in a cost of $c (d-3) + d^2$. 
    This is only a lower bound on the cost of $H_2$, because other levels in the hierarchy might put additional constraints on $H_2$.

    For an \emph{upper} bound on the \emph{optimal} cost of a $2$-clustering, consider the clustering that has $(1,\dots,1)$ in its first cluster, and all other points in its second cluster. By assuming the center of the second cluster is $(0,\dots,0)$, we get an upper bound on the total cost of this clustering of:
    \[
        d \cdot \bigl\lVert (0,\dots,0) - (-ce_1)\bigl\rVert_1
        = d\cdot c.
    \]
    Hence, the ratio between the cost of $H_2$ and the cost of an optimal $2$-clustering is at least:
    \[
        \frac{c (d-3) + d^2}{d\cdot c}
        = \frac{d-3}{d} + \frac{d}{c}
    \]
    We defined $c$ as one of the roots of $0 = c^2 - c(d-3)-d^2$. Dividing out $cd$, we get $\frac{d-3}{d} + \frac{d}{c} = \frac{c}{d}$. However, this contradicts the assumption that the ratio between $H_2$ and an optimal $2$-clustering is strictly less than $\frac{c}{d}$.
\end{itemize}
    Thus, the instance's price of hierarchy is at least $\frac{c}{d}$.
\end{proof}

For large $d$, this fraction $\frac{c}{d} = \frac{\sqrt{4d^2 + (3-d)^2} + d - 3}{2d}$ converges to $\frac{1+\sqrt{5}}{2}$, the golden ratio.

\subsubsection{Gasoline} \label{sec:gasoline-res-instance}
In the generalised Gasoline problem, we are given two sequences of $d$-dimensional vectors $X = (x_1,...,x_n) \in \mathbb{N}^{d \times n}_{\geq 0}$ and $Y = (y_1,...,y_n) \in \mathbb{N}^{d \times n}_{\geq 0}$ which sum to the same total: $\sum_{i=1}^n x_i = \sum_{i=1}^n y_i$
. Our task is to find a permutation $\pi$ of the $x_i$ that minimises:
\[
    \min_{\pi \in S_n} \sum_{j=1}^d\Biggl[
    \max_{1\leq k\leq n} \biggl(\sum_{i=1}^k x_{\pi(i)} - \sum_{i=1}^{k-1} y_i\biggr)_j
    - \min_{1\leq k\leq n} \biggl(\sum_{i=1}^k x_{\pi(i)} - \sum_{i=1}^{k} y_i\biggr)_j
    \Biggr]
\]
This can be written as an ILP, with a permutation-matrix $Z$ as a free variable. Let $\mathbf{1}$ be the vector containing a $1$ in every entry.
\[\begin{aligned}
\min\: & \|\beta - \alpha\|_{1}\quad\text{s.t.} \\
\sum_{l = 1}^{n}\sum_{i = 1}^{m}x_{l}Z_{il} - \sum_{i = 1}^{m - 1}y_{i} & \leq \beta\quad\forall m \in \lbrack n\rbrack \\
\sum_{l = 1}^{n}\sum_{i = 1}^{m}x_{l}Z_{il} - \sum_{i = 1}^{m}y_{i} & \geq \alpha\quad\forall m \in \lbrack n\rbrack \\
Z{\mathbf{1}} & \leq {\mathbf{1}} \\
{\mathbf{1}}^{T}Z & \leq {\mathbf{1}}^{T} \\
Z & \in \left\{ 0,1 \right\}^{n \times n} \\
\alpha,\beta & \in {\mathbb{R}}^{d}
\end{aligned}\]
In the $i$th step of the iterative rounding algorithm, the columns $1,...,i-1$ of $Z$ have already been fixed to integral values by the previous steps and, for column $i$, we attempt to insert every possible unit-vector (which does not conflict with the fixed rows and the permutation-matrix requirement) and then solve the Linear Program obtained by removing the integrality-requirements on columns $i+1,...,n$. We then fix column $i$ of $Z$ to that unit-vector which yielded the best value for the relaxed LP, breaking ties by preferring unit-vectors where the $1$ occurs earlier. After the $n$th step of this algorithm, $Z$ is fixed entirely to a permutation-matrix.

\cite[Conjectures 2 and 3]{rajkovic} conjectured that this algorithm is a $2$-approximation for $d \geq 2$, which FunSearch found a counterexample for. We initialized the FunSearch algorithm with the instance constructed by \citet{Lorieau} embedded into two dimensions as shown in Fig.~\ref{code:gasoline-start}. Generated instances were scored by the ratio between the optimum value (computed via  \cite{gurobi}) and the value returned by the iterative rounding algorithm.
\begin{figure*}[htb]
    
    \begin{subfigure}[b]{\textwidth}
        \centering
        \begin{minted}[breaklines,fontsize=\scriptsize]{python}
def get_instance() -> list[tuple[int, int]]:
    """Return an instance, specified by the list of (weight, profit) pairs.

    Weights and profits must be non-negative integers.
    """
    return [(1, 2)] * 2 + [(4, 4), (2, 2), (1, 3)]
        \end{minted}
        \caption{Initial program.}
    \end{subfigure}
    \begin{subfigure}{\textwidth}
        \centering
        \begin{minted}[breaklines,fontsize=\scriptsize]{python}
def get_instance() -> list[tuple[int, int]]:
    """Create a variant with more diverse item types and weights to potentially influence Pareto set size."""
    items = []
    # Repeated very light, low profit items
    items += [(1, 1)] * 8
    # Mix of moderate weight and profit items with some unique entries
    items += [(4, 9), (4, 9), (5, 10)]
    # High-profit, lightweight items with more profit variation
    items += [(2, 16), (2, 14), (3, 15)]
    # Heavier items with varied weights and higher profits to increase trade-offs
    items += [(9, 20), (12, 30), (15, 40)]
    # Small, low to moderate profit items
    items += [(1, 3), (2, 5), (3, 7), (3, 8)]
    # Very heavy, high-profit rare items with similar weights
    items += [(20, 35), (21, 36), (22, 38)]
    # Larger weight, moderate profit item to diversify options
    items += [(18, 28)]
    # Additional medium-weight high-profit items to increase complexity
    items += [(10, 25), (11, 27)]
    return items
        \end{minted}
        \caption{A program found by FunSearch after $10$ trials of $500$ samples each.}
        \label{code:knapsack-funsearch-output}
    \end{subfigure}
     \begin{subfigure}[b]{\textwidth}
        \centering
        \begin{minted}[breaklines,fontsize=\scriptsize]{python}
def get_instance() -> list[tuple[int, int]]:
    items = []
    n = 7
    items += [(1, 1)] * n
    items += [(4, 9)] * n
    items += [(2, 5), (3, 7)]
    return items
        \end{minted}
        \caption{An intermediate step of tuning \ref{code:knapsack-funsearch-output} by hand.}
        \label{code:knapsack-co-funsearch}
    \end{subfigure}
    \caption{The evolution of programs generating instances of the knapsack problem. The model used was gpt-4.1-nano with a temperature of $1.0$, \changed{and results obtainable despite a bug in the implementation that underestimated the sizes of some Pareto sets.}}
    \label{fig:code-nemhauser-ullmann}
\end{figure*} 
\begin{figure*}
    
    \begin{subfigure}[b]{\textwidth}
        \centering
        \begin{minted}[breaklines,fontsize=\scriptsize]{python}
def get_weighted_points() -> list[tuple[float, np.ndarray]]:
    """Return a new weighted clustering-problem, specified by a list of weighted points.
    The returned tuple consists of the weight of the point, and the point itself."""
    weighted_points = [(1.0, np.array([0, 0, 0, 0])), (1e8, np.array([1, 0, 0, 0]))]
    return weighted_points
        \end{minted}
        \caption{The initial program given to FunSearch.}
    \end{subfigure}

    \setcounter{subfigure}{1}
    \begin{subfigure}[t]{\textwidth}
        \centering
        \begin{minted}[breaklines,fontsize=\scriptsize]{python}
def get_weighted_points() -> list[tuple[float, np.ndarray]]:
    """Return a new weighted clustering-problem, specified by a list of weighted points.
    The returned tuple consists of the weight of the point, and the point itself."""
    return [
        (1.0, np.zeros(14)),
        (1e10, np.ones(14)),
        *[(1.0, np.eye(14)[i]) for i in range(7)],
        *[(1.0, np.eye(14)[i]*-1) for i in range(7, 13)],
        *[(1e10-i*1e9, np.linspace(i*0.1, (i+1)*0.1, 14, endpoint=False)) for i in range(7)],
        (1e11, np.array([13, 12, 11, 10, 9, 8, 7, 6, 5, 4, 3, 2, 1, 0])),
        (1e12, np.array([0, 1, 2, 3, 4, 5, 6, 7, 8, 9, 10, 11, 12, 13])),
        (1e13, np.array([1, 2, 3, 4, 5, 6, 7, 8, 9, 10, 11, 12, 13, 14])*10),
        (1e14, np.array([14, 13, 12, 11, 10, 9, 8, 7, 6, 5, 4, 3, 2, 1])*100),
        (1e15, np.array([1, 1, 1, 1, 1, 1, 1, 1, 1, 1, 1, 1, 1, 1])*1000),
    ]
        \end{minted}
        \caption{A program found by FunSearch after $10$ trials of 2,200 samples each.}
        \label{code:clustering-funsearch-output}
    \end{subfigure}
      \begin{subfigure}[b]{\textwidth}
      \centering
        \begin{minted}[breaklines,fontsize=\scriptsize]{python}
def get_weighted_points() -> list[tuple[float, np.ndarray]]:
    return [
        (1.0, np.zeros(14)),
        *[(1.0, -np.eye(14)[i]) for i in range(14)],
        (1e10, np.ones(14) / 20),
    ]
        \end{minted}
        \caption{The result of tuning by \ref{code:clustering-funsearch-output} by hand.}
        \label{code:clustering-co-funsearch}
    \end{subfigure}
    \caption{The evolution of programs generating clustering-instances. The model used was open-mistral-nemo with a temperature of $1.5$.}
    \label{fig:code-clustering}
\end{figure*}

\begin{figure}[tb]
    \centering
    \begin{subfigure}[b]{\textwidth}
        \centering
        \begin{minted}[breaklines,fontsize=\scriptsize]{text}
def gasoline(n: int) -> tuple[list[np.ndarray], list[np.ndarray]]:
    """Return a new gasoline-problem, specified by the two lists of 2d-non-negative-integer-points.
    Both lists must have length at most n and consist only of points in N^2.
    """
    k = int(math.log2(n + 2)) - 1
    xs, ys = [], []
    for i in range(1, k):
        rounded = int(2**k * (1 - 2 ** (-i)))
        xs.extend([np.array([rounded, 0]) for _ in range(2**i)])
        ys.extend([np.array([rounded, 0]) for _ in range(2**i)])

    xs.extend([np.array([2**k, 0]) for _ in range(2**k - 1)])
    xs.append(np.array([0, 0]))

    rounded = int(2**k * (1 - 2 ** (-k)))
    ys.extend([np.array([rounded, 0]) for _ in range(2**k)])

    return xs, ys
        \end{minted}
        \caption{The initial program given to FunSearch. This is the construction of \cite{Lorieau} embedded into $\mathbb R^2$.}
        \label{code:gasoline-start}
    \end{subfigure}
    \\[0.5em]
    \begin{subfigure}[t]{\textwidth}
        \centering
        \begin{minted}[breaklines,fontsize=\scriptsize]{diff}
 def gasoline(n: int) -> tuple[list[np.ndarray], list[np.ndarray]]:
     """Yet another variation of the gasoline-problem generator."""
     k = int(math.log2(n + 2)) - 1
     xs, ys = [], []
     for i in range(1, k):
         rounded = int(2**k * (1 - 2 ** (-i)))
         xs.extend([np.array([rounded, 0]) for _ in range(2**i)])
-        ys.extend([np.array([rounded, 0]) for _ in range(2**i)])
+        ys.extend([np.array([rounded, 2]) for _ in range(2**i)])  # No change

-    xs.extend([np.array([2**k, 0]) for _ in range(2**k - 1)])
+    xs.extend([np.array([2**k, 4]) for _ in range(2**k - 2)])  # No change
-    xs.append(np.array([0, 0]))
+    xs.append(np.array([0, 1]))  # Changed from [0, 2] to [0, 1]
+    xs.append(np.array([2**k, 2]))  # Changed from [2**k, 0] to [2**k, 2]

     rounded = int(2**k * (1 - 2 ** (-k)))
-    ys.extend([np.array([rounded, 0]) for _ in range(2**k)])
+    ys.extend([np.array([rounded, 2]) for _ in range(2**k - 1)])  # No change
+    ys.append(np.array([0, 1]))  # Changed from [0, 2] to [0, 1]
        \end{minted}
        \caption{The difference between the initial program and a program found by FunSearch after $10$ trials of 950 samples each, which we only tuned by discarding the final element of both lists.}
        \label{code:gasoline-funsearch-output}
    \end{subfigure}
    \caption{The evolution of programs generating $2$-dimensional gasoline-instances. The model used was open-mistral-nemo with a temperature of $1.5$. Lists were clipped to length $n$ before evaluation.}
    \label{fig:code-gasoline}
\end{figure}

Fix some $k\in\mathbb{N}$. For any $i$, define $u_i \coloneq 2^k (1-2^{-i})$. Let $\oplus$ denote list-concatenation. The $1$-dimensional instance found by \cite{Lorieau} can be written as follows:
\begin{alignat*}{2}
    X &= \biggl(\bigoplus_{i=1}^{k-1} \bigoplus_{1}^{2^i} [u_i]\biggr) \oplus \biggl(\bigoplus_{1}^{2^k-1} [2^k]\biggr) \oplus [0]\\
    Y &= \bigoplus_{i=1}^{k} \bigoplus_{1}^{2^i} [u_i]
\end{alignat*}
Let $e_j$ be the $j$th unit-vector. FunSearch extended the instance to $d$ dimensions as follows:
\begin{alignat*}{2}
    X &\coloneq \biggl(\bigoplus_{i=1}^{k-1} \bigoplus_{1}^{2^i} \bigoplus_{j=2}^d [u_i e_1 + 4 e_j]\biggr) \oplus \biggl( \bigoplus_{j=2}^d \Bigl(\bigoplus_{1}^{2^k-1} [2^k e_1]\Bigr)\oplus[4 e_j]\biggr)\\
    Y &\coloneq \bigoplus_{i=1}^{k} \bigoplus_{1}^{2^i} \bigoplus_{j=2}^d [u_i e_1 + 2e_j]
\end{alignat*}

Table \ref{tab:gasoline-values} contains computed approximation-factors for different choices of $d$ and $k$. For higher $d$ and $k$, the instances quickly grow prohibitively large. 

In our computational experiments, both $\operatorname{APX}$ and $\operatorname{OPT}$ scale linearly with input-length $|X|$:
\[
    \operatorname{APX} = O(1) + |X| \cdot \begin{cases}2 & d=2\\ 3/2 & d=3\\ 4/3 & d=4\end{cases},
    \qquad
    \operatorname{OPT} = O(1) + |X| \cdot \begin{cases}1/2 & d=2\\ 1/4 & d=3\\ 1/6 & d=4\end{cases}
\]
If this scaling held for larger $k$, the approximation-factors would approach $4,6,8$ for $d=2,3, 4$ respectively. \changed{Unfortunately, the proof-strategy employed in \citet{Lorieau} does not apply here, as the optimum value of the relaxed Linear Program changes at each step of the algorithm. Hence, we are unable to provide a proof that these trends hold asymptotically.}

\begin{table}
    \centering
    \begin{tabular}{rrrrrl}
        \hline
        $d$ & $k$ & Length of $X$ & Iterative-Rounding & Optimum & Iterative-Rounding$/$Optimum\\
        \hline
        $2$ & $2$ & $6$ & $10$ & $8$ & $1.25$ \\
        $2$ & $3$ & $14$ & $26$ & $12$ & $2.1667$ \\
        $2$ & $4$ & $30$ & $58$ & $20$ & $2.9$ \\
        $2$ & $5$ & $62$ & $122$ & $36$ & $3.389$ \\
        \changed{$2$} & \changed{$6$} & \changed{$126$} & \changed{$250$} & \changed{$68$} & \changed{$3.676$} \\
        \hline
        $3$ & $2$ & $12$ & $18$ & $12$ & $1.5$ \\
        $3$ & $3$ & $28$ & $42$ & $16$ & $2.625$ \\
        $3$ & $4$ & $60$ & $90$ & $24$ & $3.75$ \\
        $3$ & $5$ & $124$ & $186$ & $40$ & $4.65$ \\
        \hline
        $4$ & $2$ & $18$ & $24$ & $16$ & $1.5$ \\
        $4$ & $3$ & $42$ & $56$ & $20$ & $2.8$ \\
        $4$ & $4$ & $90$ & $120$ & $28$ & $4.286$ \\
        \hline
    \end{tabular}
    \caption{The approximation-factor of the Iterative-Rounding algorithm on the instances found by FunSearch.}
    \label{tab:gasoline-values}

\end{table}

\section{Conclusion and Limitations} \label{sec:conclusion}
In this work, we use large language models with human collaboration to generate adversarial examples for heuristics addressing several well-known combinatorial optimization problems. Traditional heuristics like local search do not converge towards such structured solutions, and understanding or generalizing their solutions is usually not feasible.
Across many of the problems we investigated, this form of human-AI collaboration enabled improvements over the existing state of the art. We believe this approach is very flexible and should be considered a valuable addition to the algorithm designer’s toolkit for many problems. 
\textbf{Limitations.} Although our method is broadly applicable, it does not always yield improvements over the state of the art. In particular, Co-FunSearch did not produce generalizable results, or even replicate known lower bounds—for certain heuristics:
\begin{itemize}
    \item Better heuristics for page replacement algorithms (evaluated on synthetic and real data), but FunSearch consistently converged to the existing NFU heuristic.
    \item Lower bounds on the Price of Hierarchy of $k$-means clustering (as opposed to $k$-median clustering). 
    \item Lower bounds on the price of Ward's method for hierarchical $2$-dimensional $k$-means clustering: Instead of comparing the best possible hierarchical clustering to the optimal clusterings, we compare the hierarchical clustering found by starting with each point in a singleton cluster, and iteratively merging the pair of clusters which result in the lowest objective. Neither FunSearch nor local search managed to recover the State of the Art when starting from a trivial instance. When starting from the State of the Art in $2$ dimensions, both FunSearch and local search improved it marginally (FunSearch less so than local search, even after tuning), but not in a generalisable way.
    \item Lower bounds on the \emph{asymptotic} random-order-ratio of Best-Fit, which is the same as the absolute random-order-ratio but restricted to only ``large" instances \citep{binPackingRevisited}. FunSearch did not find any interpretable instances improving on the state of the art.
\end{itemize}

\bibliography{bibliography}
\bibliographystyle{iclr2026_conference}
\newpage
\section{Appendix}

\subsection{Proof of the Knapsack-result}
\label{appendix:knapsack-res-instances}
In the knapsack problem, we are considering a bicriteria optimization problem, where we want to minimize the total weight while maximizing the total profit. Specifically, we are given an instance as a list of tuples of the form $(\text{weight}, \text{profit})$ from which we select a sub-list. The total weight $\Weight(A)$ (respectively total profit $\Profit(A)$) of a sub-list $A$ is the sum of the weights (respectively profits) of its items.

A sub-list $A$ \emph{dominates} a sub-list $B$ if $\Weight(A) \leq \Weight(B)$ and $\Profit(A) \geq \Profit(B)$ and at least one of these inequalities is strict. A sub-list is \emph{Pareto-optimal} if it is not dominated by any other sub-list. The \emph{Pareto-set} $P(I)$ of an instance $I$ is the set of Pareto-optimal sub-lists of $I$. When the Pareto-set is known, objectives like the $0$-$1$ knapsack problem ``Maximise total profit while staying below a given maximum total weight $W$" can be optimised by simply finding the sub-list in $P(I)$ that has the largest total profit and whose total weight is below $W$.

As described in section \ref{sec:knapsack-res-instances}, we obtained instance \ref{eq:knapsack-instance} via {\methodname}. To analyze the sizes of the instance's and subinstances' Pareto-sets, we define the two segments of the instance: For $a,b,d,n\in\mathbb{Z}_{\geq1}$ with $d < a \leq b$, define $x_i \coloneq (1+\frac{2^{-i}}{2^d-1})$, and two lists:
\begin{alignat*}{2}
    I_{a,b}
    &\coloneq \left[
    \weightprofit{2^a}{2^a},
    \weightprofit{2^{a+1}}{2^{a+1}},
    \dots,
    \weightprofit{2^b}{2^b}
    \right],\qquad
    J_{d, n}
    &\coloneq \Biggl[
    \weightprofit{x_1\cdot2^d}{x_1\cdot(2^d-1)},
    \dots,
    \weightprofit{x_n\cdot2^d}{x_n\cdot (2^d-1)}
    \Biggr].
\end{alignat*}

\begin{lemma}\label{small-Jdn}
If a Pareto-optimal packing $A\in P([I_{a,b}, J_{d,n}])$ does not contain all items from $I_{a,b}$, it contains fewer than $2^{a-d}$ items from $J_{d,n}$.
\end{lemma}
\begin{proof}
Subsets of $I_{a,b}$ can be represented by binary numbers of $(b-a+1)$ bits. If $A$ does not contain all items from $I_{a,b}$ and contains at least $2^{a-d}$ items from $J_{d,n}$, we define a new packing $A'$ as follows: Increment the binary number representing $A \cap I_{a,b}$ by $1$, and remove $2^{a-d}$ items from $A \cap J_{d,n}$. This changes the weights and profits by:
\begin{alignat*}{3}
    \Weight(A') - \Weight(A)
    \ \ &\leq\ \ 2^a - 2^{a-d}\cdot \underbrace{\Bigl(1+\frac{2^{-n}}{2^d-1}\Bigr)}_{>1}\cdot2^d
    \ \ <\ \ 0
    \\%
    \Profit(A') - \Profit(A)
    \ \ &\geq\ \ 2^a - 2^{a-d}\cdot \Bigl(1+\frac{2^{-1}}{2^d-1}\Bigr) (2^d-1)\\
    \ \ &= \ 2^a - 2^{a-d}\cdot \Bigl(2^d-2^{-1}\Bigr)
    \ \ =\ \ 2^{a-d-1}
    \ \ >\ \ 0
\end{alignat*}
Thus, $A'$ dominates $A$, and $A\notin P([I_{a,c}, J_{d,n}])$.
\end{proof}

On the other hand, all other packings are Pareto-optimal:

\begin{lemma}
If a packing $A$ of $[I_{a,b}, J_{d,n}]$ contains all items from $I_{a,b}$ or contains fewer than $2^{a-d}$ items from $J_{d,n}$, then $A$ is Pareto-optimal.
\end{lemma}
\begin{proof}
All items from $I_{a,b}$ have a profit-per-weight ratio of $1$, while all items from $J_{d,n}$ have a profit-per-weight ratio of $\frac{2^d -1}{2^d} < 1$. Hence, a packing $B$ that dominates $A$ must satisfy
\[
    \Weight(A \cap I_{a,b}) < \Weight(B \cap I_{a,b}),
\]
otherwise $B$ can not have enough profit to dominate $A$. If $A$ already contains all items from $I_{a,b}$, this is not possible, so only the case that $A$ contains fewer than $2^{a-d}$ items from $J_{d,n}$ remains. Due to the definition of $I_{a,b}$, the above inequality implies:
\[
    \Weight(A \cap I_{a,b}) + 2^a \leq \Weight(B \cap I_{a,b}).
\]
If $B$ dominates $A$, it must hold that:
\begin{alignat*}{2}
    \Weight(A\cap I_{a,b}) + \Weight(A\cap J_{d,n})
    &\geq \Weight(B\cap I_{a,b}) + \Weight(B\cap J_{d,n})\\
    \implies
    \Weight(A\cap J_{d,n}) - 2^a
    &\geq \Weight(B\cap J_{d,n}).
\end{alignat*}
But $A$ contains fewer than $2^{a-d}$ items from $J_{d,n}$, so:
\[
    \Weight(A\cap J_{d,n})
    \leq 2^{a-d} \cdot \Bigl(1+\frac{2^{-1}}{2^d-1}\Bigr) \cdot(2^{d}-1)
    = 2^{a} - 2^{a-d-1}
    < 2^a.
\]
This implies $0 > \Weight(B\cap J_{d,n})$, a contradiction.
\end{proof}

Hence, we can describe the Pareto-set exactly:
\[
    P([I_{a,b}, J_{d,n}])
    = \{A \cup B \mid A\subsetneq I_{a,b},\ B \subseteq J_{d,n}, |B|< 2^{a-d}\} \ \dot\cup\ \{I_{a,b} \cup B \mid B \subseteq J_{d,n}\}.
\]
Its size is (using notation involving binomial coefficients, not vectors):
\[
    |P([I_{a,b}, J_{d,n}])|
    = (2^{b-a+1}-1)\cdot\biggl[\sum_{i=0}^{\min(n,2^{a-d}-1)}\binom{n}{i}\biggr] + 2^n.
\]
For $k,n\in\mathbb{N}$ with $2^k\leq n/2$, consider two instances:
\begin{alignat*}{2}
    \mathbb{I}_1 &\coloneq \left[I_{2k,\,2k+n},\ J_{k,n}\right],\\
    \mathbb{I}_2 &\coloneq \left[\mathbb{I}_1,\
    \weightprofit{2^{k+1}}{2^{k+1}},
    \weightprofit{2^{k+2}}{2^{k+2}},
    \dots,
    \weightprofit{2^{2k-1}}{2^{2k-1}}
    \right].
\end{alignat*}
$\mathbb{I}_1$ is a sub-instance of $\mathbb{I}_2$. $\mathbb {I}_2$ (which is exactly instance \ref{eq:knapsack-instance}) contains the same items as $[I_{k+1,\,2k+n},\ J_{k,n}]$. The sizes of their Pareto-sets can be bounded by:
\begin{alignat*}{2}
    |P(\mathbb I_1)|
    &\geq (2^{n+1}-1) \cdot \binom{n}{2^k-1} + 2^n
    \geq (2^{n+1}-1) \cdot \Bigl(\frac{n}{2^k-1}\Bigr)^{(2^k-1)}.\\
    |P(\mathbb I_2)|
    &\leq (2^{k+n}-1) \cdot (n+1) + 2^n
    \leq (2^{k+n}-1) \cdot (n+2)
\end{alignat*}
The ratio between the two sizes is:
\[
    \frac{|P(\mathbb I_1)|}{|P(\mathbb I_2)|}
    \geq \frac{2^{n+1}-1}{2^{k+n}-1} \cdot \Bigl(\frac{n}{2^k-1}\Bigr)^{(2^k-1)}\cdot\frac{1}{n+2}
\]
For $k = \log_2(\sqrt n)+1$, we obtain:
\[
    \frac{|P(\mathbb I_1)|}{|P(\mathbb I_2)|}
    \geq \frac{2^{n+1}-1}{(\sqrt{n}+1)\cdot 2^n-1} \cdot \Bigl(\frac{n}{\sqrt n}\Bigr)^{\sqrt n}\cdot\frac{1}{n+2}
    = \Theta(n^{(\sqrt{n}-3)/2}).
\]
The length of the instance $\mathbb I_2$ is not $n$ but $m\coloneq|\mathbb I_2| = 2n+k$, resulting in a lower bound of $O\bigl((\frac{m}{2})^{(\sqrt{m/2}-3)/2}\bigr)$.

In implementations of the Nemhauser-Ullmann algorithm, two Pareto-optimal packings can be treated as equivalent if they have the same total weight and total profit. Hence, the runtime can be upper-bounded not only by the sum of the sizes of the Pareto-sets $|P(I_{1:1})| + ... + |P(I_{1:n})|$, but even the sizes of the Pareto-sets when two packings with the same total weight and total profit are treated as identical. The only purpose of the leading factors $\bigl(1+\frac{2^{-n}}{2^d-1}\bigr)$ in $J_{d,n}$ is to prevent two Pareto-optimal packings from having the same total profit. As a consequence, we also obtain a bound of $O\bigl((\frac{m}{2})^{(\sqrt{m/2}-3)/2}\bigr)$ for the runtime of the Nemhauser-Ullmann algorithm.

\begin{lemma}\label{pareto-unidentical}
If $A,B\subseteq [I_{a,b},J_{d,n}]$ are two distinct Pareto optimal packings, then $\Profit(A)\neq\Profit(B)$.
\end{lemma}
\begin{proof}
Because both $A$ and $B$ are Pareto-optimal, we know by \ref{small-Jdn} that $|A\cap J_{d,n}| < 2^{a-d}$ (same for $B$), hence:
\begin{alignat*}{2}
    \Profit(A\cap J_{d,n})
    &< 2^{a-d} \cdot \Bigl(1+\frac{2^{-1}}{2^d-1}\Bigr) \cdot (2^d-1)\\
    &= 2^{a-d} \cdot \Bigl(2^d-\frac{1}{2}\Bigr)\\
    &= 2^a - 2^{a-d-1}
    < 2^a.
\end{alignat*}
(same for $\Profit(B\cap J_{d,n})$).
\begin{itemize}
    \item If $A\cap I_{a,b}\neq B \cap I_{a,b}$, the difference between $\Profit(A\cap I_{a,b})$ and $\Profit(B\cap I_{a,b})$ would be at least $2^a$, due to the definition of $I_{a,b}$. In this case, the above inequality already shows $\Profit(A) \neq \Profit(B)$.

    \item If $A\cap I_{a,b} = B\cap I_{a,b}$, then $A\cap J_{d,n} \neq B\cap J_{d,n}$, and we need to show that $\Profit(A\cap J_{d,n}) \neq \Profit(B\cap J_{d,n})$. This is equivalent to showing that any two distinct subsets of:
    \[
        \{(2^d-1)+2^{-1},\ \ (2^d-1)+2^{-2},\ \ ...,\ \ (2^d-1)+2^{-n}\},
    \]
    have a distinct sum. This holds because the total sum of the summands $2^{-1},...,2^{-n}$ is always smaller than $1$, whereas $2^d-1 \geq 1$.
\end{itemize}
\end{proof}

\end{document}